\documentclass{article}

\usepackage{times}
\usepackage{graphicx} 

\usepackage{natbib}

\usepackage{algorithm}
\usepackage{algorithmic}
\usepackage{sidecap}
\usepackage{hyperref}

\usepackage{wrapfig,subfigure}
\usepackage{mathrsfs}
\usepackage{amsmath, amsthm, amssymb, multirow, paralist}
\usepackage{url}
\usepackage{algorithm,algorithmic}
\newtheorem{thm}{Theorem}
\newtheorem{prop}{Proposition}
\newtheorem{lemma}[thm]{Lemma}

\newtheorem{ass}[thm]{Assumption}

\def \y {\mathbf{y}}
\def \E {\mathrm{E}}

\def \g {\mathbf{g}}

\def \u {\mathbf{u}}

\def \w {\mathbf{w}}
\def \R {\mathbb{R}}

\def \A {\mathcal{A}}

\def \a {\mathbf{a}}

\def \B {\mathbf{B}}

\def \wh {\widehat{\w}}

\def \gh {\widehat{\g}}

\def \V {\mathcal{V}}

\def \P {\mathbb{P}}
\def \T {\mathcal{T}}

\def \Rt {\mathcal{R}}
\def \V {\mathcal{V}}

\usepackage{booktabs}

\usepackage[accepted]{icml2016}

\icmltitlerunning{Optimal Dynamic Regrets of Online Learning with True and Noisy Gradient}

\begin{document} 

\twocolumn[
\icmltitle{Tracking Slowly Moving Clairvoyant: \\Optimal Dynamic Regret of Online Learning with True and Noisy Gradient}

\icmlauthor{Tianbao Yang}{tianbao-yang@uiowa.edu}
\icmladdress{Department of Computer Science, The University of Iowa, Iowa City, IA 52242, USA}
\icmlauthor{Lijun Zhang}{zljzju@gmail.com}
\icmladdress{National Key Laboratory for Novel Software Technology, Nanjing University, Nanjing, China}
\icmlauthor{Rong Jin}{jinrong.jr@alibaba-inc.com}
\icmladdress{Institute of Data Science and Technologies at Alibaba Group, Seattle, USA}
\icmlauthor{Jinfeng Yi}{jinfengy@us.ibm.com}
\icmladdress{IBM Thomas J. Watson Research Center, Yorktown Heights, NY 10598, USA}

\icmlkeywords{boring formatting information, machine learning, ICML}

\vskip 0.3in
]

\begin{abstract} 
This work focuses on dynamic regret  of online convex optimization  that compares the performance of online learning to a clairvoyant who knows the sequence of loss functions in advance and hence selects the minimizer of the loss function at each step. By assuming that the clairvoyant moves slowly (i.e., the minimizers change slowly), we present several improved variation-based  upper bounds of the dynamic regret under the true and noisy gradient feedback, which are  {\it optimal} in light of the presented  lower bounds. The key to our analysis is to explore a regularity metric that measures the temporal changes in the clairvoyant's minimizers, to which we refer as {\it path variation}. Firstly, we present a general lower bound in terms of the path variation, and then show that under full information or gradient feedback we are able to achieve an optimal dynamic regret. Secondly, we present a lower bound with noisy gradient feedback and then show that we can achieve optimal dynamic regrets under a stochastic gradient feedback and two-point bandit feedback. Moreover, for a sequence of  smooth loss functions that admit a small variation in the gradients, our dynamic regret under the two-point bandit feedback matches what is achieved with full information. 
\end{abstract} 
\setlength{\belowdisplayskip}{5pt} \setlength{\belowdisplayshortskip}{1pt}
\setlength{\abovedisplayskip}{5pt} \setlength{\abovedisplayshortskip}{1pt}

\section{Introduction}\label{sec:intro}
Online convex optimization (OCO) can be deemed as a repeated game between an online player and an adversary, in which an online player iteratively chooses a decision and then her  decisions incur  (possibly different) losses by the loss functions chosen by the adversary. These loss functions are unknown to the decision maker ahead of time, and can be adversarial or even depend on the action taken by the decision maker. To formulate the problem mathematically, let $\Omega\subseteq \R^d$ denote a convex decision set (i.e., the feasible set of the decision vector), $\w_t\in\Omega$ denote the decision vector and $f_t(\cdot):\R^d\rightarrow \R$ denote the loss function at the $t$-th step, respectively. The goal of the online learner is to minimize her cumulative loss $\sum_{t=1}^Tf_t(\w_t)$. The traditional  performance metric -  the regret of the decision maker, is defined as the difference between the total cost she has incurred and that of the best fixed decision in hindsight, i.e., 
\begin{align}\label{eqn:static}
\sum_{t=1}^Tf_t(\w_t) - \min_{\w\in\Omega}\sum_{t=1}^Tf_t(\w).
\end{align}

Recently, there emerges a surge of interest~\citep{besbes-2013-optimization,DBLP:journals/corr/HallW13,DBLP:conf/aistats/JadbabaieRSS15} in the dynamic regret that compares the performance of online learning to a sequence of optimal solutions. If we denote by $\w_t^* \in \Omega$ an optimal solution of $f_t(\w)$,  the dynamic regret is defined as
\begin{align}\label{eqn:dynamic}
\sum_{t=1}^T(f_t(\w_t)-f_t(\w_t^*))=\sum_{t=1}^T( f_t(\w_t)-\min_{\w\in\Omega}f_t(\w))
\end{align}
 i.e., the performance of the online learner is compared to a clairvoyant who knows the sequence of loss functions in advance, and hence selects the minimizer $\w_t^*$ at each step. Compared to the traditional regret in~(\ref{eqn:static}) (termed as static regret), the dynamic regret is more aggressive since the performance of the clairyomant in the dynamic regret model is always better than that in the static regret model, i.e., $\sum_{t=1}^Tf_t(\w_t^*)\leq \min_{\w\in\Omega}\sum_{t=1}^Tf_t(\w)$. It was pointed out that algorithms that achieve performance close to the best fixed decision may perform poorly in terms of dynamic regret~\citep{besbes-2013-optimization}.

\begin{table*}[t]
\centering\caption{Summary of dynamic regret bounds in this paper and comparison with the previous work. N.B. means that no better bounds are explicitly given. However, the bounds in the degenerate case may apply.  N.A. means that not available.  $^*$ marks the result is restricted to a family of smooth loss functions with vanishing gradients in the feasible domain. Please note that the bandit feedback in this work refers to two-point bandit feedback, and that in \citep{besbes-2013-optimization} refers to  noisy one-point bandit feedback. }
\label{tab:data}
\begin{tabular}{ll|lll}
\toprule
&&This paper&\citep{besbes-2013-optimization}&\citep{DBLP:conf/aistats/JadbabaieRSS15}\\
 \midrule
Loss function&Feedback&path variation &functional variation& three variations\\
\midrule
Lipschitz&  Full Information &$O(V^p_T)$&$O(V^f_T)$&\hspace*{-0.7in}$O(\min(\sqrt{V^p_TV^g_T}, (V^g_T)^{1/3}T^{1/3}(V^f_T)^{1/3}))$\\
Lipschitz & True Gradient  & N.B.&N.B.&$O(\sqrt{V^g_TV^p_T})$\\ 
Smooth  & True Gradient  & $O(V^p_T)^*$&N.B.&N.B.\\ 
\midrule
Lipschitz& Stochastic Grad. & $O(\sqrt{V^p_TT})$&$O((V^f_T)^{1/3}T^{2/3})$&N.A.\\ 
Lipschitz& Bandit & $O(\sqrt{V^p_TT})$&$O((V^f_T)^{1/5}T^{4/5})$&N.A.\\
Smooth/Linear& Bandit & $O(\max\{\sqrt{V^p_TV^g_T}, V^p_T\})$&N.B.&N.A.\\
\midrule
\multicolumn{2}{c|}{Lower Bounds} &Yes& Yes&No\\
\bottomrule
\end{tabular}
\vspace*{-0.2in}
\end{table*}

Unfortunately, it is impossible to achieve a sublinear dynamic regret for {\it any} sequences of loss functions~(c.f.  Proposition~\ref{prop:2}). In order to achieve a sublinear dynamic regret, one has to impose some regularity constraints on the sequence of loss functions. In this work, we leverage a notion of variation that measures how fast the clairvoyant moves, i.e., how fast the minimizers of the sequence of loss functions change, to which we refer as {\bf path variation} in order to differentiate with other variation definitions. Formally the path variation is defined as 
\begin{align}
V^p_T \triangleq\max_{\{\w_t^*\in\Omega_t^*\}^T_{t=1}} \sum_{t=1}^{T-1} \|\w_t^* - \w_{t+1}^*\|_2
\end{align}
where $\Omega^*_t$ denotes the set of all minimizers of $f_t(\w)$ to account for the potential non-uniqueness. 
We aim to develop optimal dynamic regrets when the clairvoyant moves slowly  given  (noisy) gradient feedback (including bandit feedback) for non-strongly convex loss functions. 
The main results are summarized in Table~\ref{tab:data} and the contributions of this paper are summarized below. 
\vspace*{-0.1in}
\begin{itemize}
\item We present a general  lower bound dependent solely on $V^p_T$ and show that under true gradient feedback for smooth functions with vanishing gradients in the feasible domain, one can achieve the optimal dynamic regret of $O(V^p_T)$ comparable to that with full information feedback. 

\item We present a lower bound under a noisy (sub)gradient feedback dependent on $V^p_T$ and $T$, and then show that online gradient descent (OGD) with an appropriate step size can achieve the optimal dynamic regret of  $O(\sqrt{V^p_TT})$ under both stochastic gradient feedback and two-point bandit feedback. 

\item When the loss functions are smooth, we establish an improved  dynamic regret under the two-point bandit feedback, which could match the bound achieved with the full information feedback in a certain condition. 
\end{itemize}
\vspace*{-0.1in}
We note that a regularity metric similar to the path variation (possibly measured in different norms) has been explored in shifting regret analysis~\cite{Herbster:1998:TBE:296371.296382}  and drifting regret analysis~\cite{journals/corr/abs-1202-3323,DBLP:journals/jmlr/BuchbinderCNS12}. The regret against the shifting experts was studied in tracking the best expert, where the best sequence of minimizers are assumed to change for a constant number of times. In drifting regret analysis, the constraint is relaxed to that the path variation is small.  In fact, a similar dynamic regret bound to $\sqrt{V^p_TT}$  has been established for online convex optimization over the simplex~\cite{journals/corr/abs-1202-3323}, where the path variation is measured in $\ell_1$ norm. The present work focuses on OCO in the Euclidean space and considers noisy gradient feedback. A more general variation is considered in~\cite{DBLP:journals/corr/HallW13}, where a sequence of  (or a family of) dynamic models $\phi_1, \ldots, \phi_T$ are  revealed by the environment for the learner to predict the decision in the next step. Their variation is defined as $V^\phi_T=\sum_{t=1}^{T-1}\|\w^*_{t+1} - \phi_t(\w^*_t)\|$ for a sequence of comparators and their dynamic regret scales as $V^\phi_T\sqrt{T}$, which is worse than our bounds when $\phi_t(\w)=\w$.

There has been a different notion to measure the point-wise changes  in the sequence of loss functions that measure the changes of two consecutive functions at any feasible points. For example,  \citet{besbes-2013-optimization} considered the functional variation defined as 
\begin{align}\label{eqn:fv}
V^f_T = \sum_{t=1}^{T-1}\max_{\w\in\Omega}|f_t(\w) - f_{t+1}(\w)|.
\end{align}
Besbes et al. considered two feedback structures, i.e., the noisy gradient and the noisy cost, and established sublinear dynamic regret for both feedback structures. For  Lipschitz continuous loss functions, their results are presented in Table~\ref{tab:data} ~\footnote{For strongly convex loss functions, better bounds were also established in~\citep{besbes-2013-optimization}}. An annoying fact is that even the sequence of Lipschitz  loss functions change slowly (namely the functional variation is small), Besbes' dynamic regret is worse than $O(\sqrt{T})$, the optimal rate for static regret. In comparison, our results match that for static regret when the clairvoyant moves slowly such that the path variation is a constant. 
Another variation that measures  point-wise difference between loss functions is the gradient variation introduced in~\citep{chiang-online}, which is defined as 
\begin{align}
V_T^g\triangleq\sum_{t=1}^{T}\max_{\w\in\Omega}\|\nabla f_t(\w) - \nabla f_{t-1}(\w)\|_2^2.
\end{align}
The gradient variation has been explored for bounding the static regret~\citep{chiang-online,DBLP:conf/nips/RakhlinS13,DBLP:journals/ml/YangMJZ14}.
Recently, \citet{DBLP:conf/aistats/JadbabaieRSS15} used the three variations and developed possibly better dynamic regret than using a single variation measure for non-strongly convex loss functions. They considered the full information feedback (i.e., the whole loss function is revealed to the learner) and a true gradient feedback for a sequence of bounded functions.  Their results are also presented in Table~\ref{tab:data}. In comparison, our results could be potentially better when the clairvoyant moves slowly. Different from~\citep{DBLP:conf/aistats/JadbabaieRSS15}, we consider the noisy gradient feedback (including the bandit feedback) and develop both upper bounds and lower bounds. 

\section{Optimal Dynamic Regret with Noiseless Information}\label{sec:exam}
In this section, we present an optimal dynamic regret bound dependent solely on  $V^p_T$. We will first present a lower bound and then present optimal upper bounds in two settings: (i) the full information of the loss function is revealed at each step; (ii) only the true gradient at the decision vector is revealed for smooth  loss functions that have vanishing gradients  in the feasible domain. 

\subsection{Preliminaries and a Lower Bound}
Since it is impossible to achieve a sublinear dynamic regret for any sequence of loss functions. We consider the following family of functions that admit a path variation constraint:
\begin{align}
\V_p = \{\{f_1,\ldots, f_T\}: V^p_T\leq B_T\}
\end{align}
where $B_T$ is the budget. For a (randomized) policy $\pi$ that generates a sequence of solutions $\w_1,\ldots, \w_T$ for a sequence of loss functions $f_1,\ldots, f_T$ under the feedback structure $\phi$,  its {\bf dynamic regret}  is defined as 
\begin{align*}
\Rt^\pi_{\phi}(\{f_1,\ldots, f_T\})=\E^{\pi}\left[\sum_{t=1}^Tf_t(\w_t)\right] - \sum_{t=1}^T f_t(\w_t^*)
\end{align*}
The worst dynamic regret of $\pi$ over $f\in\V_p$ is 
\begin{align*}
\Rt^\pi_\phi(\V_p, T) = \sup_{f\in\V_p}\Rt^\pi_{\phi}(\{f_1,\ldots, f_T\})
\end{align*}
Note that the dynamic regret remains the same for different sequences of optimal solutions $\w_t^*,t=1,\ldots, T$. 

Below, we  establish a general lower bound of the dynamic regret for  the following family of policies: 
\begin{align}\label{eqn:lowerpsi}
\A=\left\{\pi: \w_t = \left\{\begin{array}{ll}\pi_1(U), & t=1\\ \pi_t(\{\phi_{\tau}(f_{\tau})\}_{\tau=1}^{t-1}, U), & t\geq 1\end{array}\right.\right\}
\end{align} 
where $\phi_t(f_t)\in\R^k$ denotes any  feedback of $f_t$, and $U\in\mathbb U$ denotes a random variable, $\pi_1:\mathbb U\rightarrow \Omega$, $\pi_t: \R^{(t-1)k}\times\mathbb U\rightarrow \Omega$ are measurable functions. 
\begin{prop}\label{prop:2}
Let $C, C_1, C_2$ be positive constants independent of $T$ and $V^p_T$, and let $\pi$ be any policy in $\A$. 
\vspace*{-0.05in}
\begin{itemize}
\item  If $B_T\geq C_1T$, then there exists a positive constant $C_2$ such that
\[
\Rt^{\pi}_\phi(\V_p,T)\geq C_2T.
\]
\item For any $\gamma\in(0,1)$, there exists a sequence of loss functions $f_1, \ldots, f_T$ and a positive constant $C$ such that $V^p_T\leq o(T)$ and 
 \[
\Rt_\phi^{\pi}(\{f_1,\ldots, f_t\}) \geq  C(V^p_T)^{\gamma}.
\]
\end{itemize}
\end{prop}
{\bf Remark: }  The first part indicates that it is impossible to achieve a sublinear dynamic regret if there is no constraint on  the sequence of loss functions. Therefore, in the sequel, we only consider $B_T\leq o(T)$.  A similar result to the first part using $V^f_T$ as the regularity measure has been established~\cite{besbes-2013-optimization}.  The second part is novel, which indicates that it is impossible to achieve a better dynamic regret bound of $O((V^p_T)^{\alpha})$ with $\alpha<1$. If otherwise, 
it then contradicts to the lower bound in the second part of Proposition~\ref{prop:2}. 
\begin{proof}
Fix $T\geq 1$ and $\gamma\in(0,1)$. To generate the sequence of loss functions, we create a sequence of random variables $\varepsilon_1, \ldots, \varepsilon_T$, where each $\varepsilon_t$ is sampled independently from $\{\sigma, -\sigma\}$ with equal probabilities. It is obvious that $\E[\varepsilon_t]=0$ and $\E[\varepsilon_t^2]=\sigma^2$. For each $\varepsilon_t$, we define a loss function $f_t(w)= \frac{1}{2}(w - \varepsilon_t)^2$. 
Assume $\sigma\in(0,1)$ whose value will be specified later. Let the feasible domain be $\Omega=[-1, 1]$. 
We have
\begin{align*}
\E&\left[\Rt^{\pi}_\phi(\{f_1,\ldots, f_t\})\right] = \E\left[\sum_{t=1}^T f_t(w_t) - f_t(\varepsilon_t) \right]\\
& = \sum_{t=1}^T\E\left[\frac{w_t^2 }{2}+ \frac{\varepsilon_t^2}{2} - w_t\varepsilon_t\right]  \geq \frac{\sigma^2}{2} T
\end{align*}
where $\E[\cdot]$ denotes the expectation over the randomness in the sequence of loss functions $f_1,\ldots, f_t$ and the policy $\pi$ and the last inequality is due to that $w_t$ is independent of $\varepsilon_t$. 
We also have
$V^p_T =  \sum_{t=1}^{T-1} |\varepsilon_t - \varepsilon_{t+1}|\leq 2\sigma T$. 
To prove the first part, we let $\sigma$ be a constant $C_1/2$, then any sequences of loss functions generated as above constitute a subset $\V_p'\subset \V_p$. Then 
\begin{align*}
\Rt^\pi_\phi(\V_p, T)&\geq \Rt^\pi_\phi(\V_p', T)\geq\E\left[\Rt^{\pi}_\phi(\{f_1,\ldots, f_t\})\right] \geq \frac{C_1^2}{8}T
\end{align*}
To prove the second part, we set $\sigma=T^{-\mu}$ with $\mu = (1-\gamma)/(2-\gamma)\in(0,1/2)$. 
Then there exits a positive constant $C$ such that 
$\E\left[\Rt_\phi^{\pi}(\{f_1,\ldots, f_T\}) -C(V^p_T)^{\gamma}\right] \geq 0$, 
which implies  that there exists a sequence of loss functions $f_1,\ldots, f_T$ such that $
\Rt_\phi^{\pi}(\{f_1,\ldots, f_T\}) \geq  C(V^p_T)^{\gamma}$.

We note that if $\gamma =1$, we have $\mu = 0$ and therefore $B_T=\Omega(T)$ which reduces to the lower bound in the first part. Therefore, we restrict $\gamma\in(0,1)$.
\end{proof}
An interesting  question   is that whether an $O(V^p_T)$ dynamic regret bound is achievable, if not what is the best we can achieve. In particular, we are interested in scenarios when the feedback $\phi_t(f_t)=\phi_t(\w_t, f_t)$ only gives a (noisy) gradient of $f_t(\w)$ at $\w_t$. 

Before delving into the noisy gradient feedback, we first show that an $O(V^p_T)$ upper bound is achievable  with full information of the loss functions or with full gradient feedback provided  that the loss functions are smooth and have vanishing gradients.  
We  make the following assumptions throughout the paper without explicitly mentioning it in the sequel. 
\begin{ass}\label{ass:1} For $\{f_1,\ldots, f_T\}\in\V_p$, there exists a $r>0$ such that $\sup_{\w_t^*\in\Omega_t^*}\|\w - \w_t^*\|_2\leq r$, for any $\w\in\Omega$ and $1\leq t\leq T$.
\end{ass}

\subsection{Online Learning with Full Information}
Assume that at each step the full information of the loss function $f_t(\w)$ is revealed after the decision $\w_t$ is submitted, and each loss function $f_t(\w)$ is G-Lipschitz continuous. Then we can update $\w_{t+1}$ by 
\[
\w_{t+1} = \min_{\w\in\Omega} f_t(\w), t\geq 1
\]
with $\w_1$ be any point in $\Omega$. To  analyze the dynamic regret, we denote by $\w_{0}^* = \w_1$.  
\begin{align*}
&\sum_{t=1}^Tf_t(\w_t) - \sum_{t=1}^Tf_t(\w_t^*) = \sum_{t=1}^Tf_t(\w^*_{t-1}) - \sum_{t=1}^Tf_t(\w^*_{t})\\
&\leq \sum_{t=1}^T G\|\w^*_{t-1} - \w_{t}^*\|_2 = G\|\w_1 - \w_1^*\|_2 \\
&+ G\sum_{t=1}^{T-1} \|\w^*_{t} - \w^*_{t+1}\|_2\leq Gr + GV^p_T = O(\max(V^p_T, 1)). 
\end{align*}
It is notable that a similar upper bound of $O(\max(V^f_T,1))$ with the full information can be achieved~\citep{DBLP:conf/aistats/JadbabaieRSS15}.

\subsection{Online Learning with Gradient Feedback}
Full information may not be  available. In practice, only some partial information of the $f_t(\w)$ regarding the decision vector $\w_t$ is available. In this subsection, we assume that only the gradient information $\nabla f_t(\w_t)$ is available after the decision $\w_t$ is submitted. Below, we will first present several examples showing that $O(V^p_T)$ is achievable and generalize the analysis to a broad family. 

We consider two loss functions  $g_1(w)=\max(w,0)^2$ and $g_2(w)=(w-\alpha)^2$ defined in the domain $\Omega=[-1,3]$ and divide all iterations $1,\ldots, T$ into a number $m$ of batches with each batch size of $\Delta_T$. Assume the adversary selects $g_1(\cdot)$ in odd batches and $g_2(\cdot)$ in even batches, and at each step  the full gradient feedback is available, i.e., $\phi_t(w_t, f_t) = f'_t(w_t)$. The example is similar to that presented in~\citep{besbes-2013-optimization} except that $g_1(w)$ is not strongly convex. Below, we consider two instances of the above example with different $\Delta_T$ and $\alpha$. For the updates, we adopt the OGD, i.e., 
\[
w_{t+1} = \Pi_{\Omega}[w_t - \eta f'_t(w_t)], \quad t=1,\ldots, T-1
\]
where $\Pi_{\Omega}[\cdot]$ denotes the projection into the domain $\Omega$. 

{\bf Instance 1.}   $\Delta_T = \lceil T/2\rceil$ and  $\alpha=1$. Then $V^p_T = 1$.

 Given the value of $\Delta_T$, there are two batches.  Let $\Gamma_1, \Gamma_2\subseteq T$ denote the iteration indices in the first and the second  batch, respectively, and let $\Gamma_j[1]$ denote the first iteration of the $j$-th batch. We adopt a constant step size $\eta=1/2$ with a starting point $w_0=0$.   Then $w_t = 0, t\in \Gamma_1$. For the first iteration $t\in\Gamma_2$ we have $w_t = \Pi_{\Omega}[0 - \eta g'_1(0)]=0$. And for all remaining  iterations $t\in\Gamma_2$, we have  $w_t = \Pi_{\Omega}[w_{t-1} - \eta g'_2(w_{t-1})]=\alpha$.  As a result $w_t = w_t^*, t\in[T]$ except $w_{\Gamma_2[1]}=0\neq w_{\Gamma_2[1]}^*$, which indicates that the dynamic regret is $f_{\Gamma_2[1]}(w_{\Gamma_2[1]}) -f_{\Gamma_2[1]}(w^*_{\Gamma_2[1]})=g_2(0) - g_2(1)= 1$.

{\bf Instance 2.}  $\theta=C/\sqrt{T}$, $\Delta_T = \lfloor 1+1/(2\theta)\rfloor$, $\alpha=1 + (1-2\theta)^{\Delta_T}$, and $T>4C^2$ (note that $\theta<1/2$ and $1\leq \alpha\leq 2$). Then 
\vspace*{-0.05in}
\begin{align*}
V^p_T=\sum_{j=1}^{m}&\alpha  \leq 2m = \frac{2T}{\Delta_T}=\frac{2T}{ \lfloor 1+\sqrt{T}/(2C)\rfloor}\leq 4C \sqrt{T} 
\end{align*}
 The example is similar to the above except that the loss functions change more frequently. 
 We consider OGD with a constant step size $\eta=1/2$ and $w_1=1$.   Similar as before, we use $j=1, \ldots, m$ to denote the batch index, $\Gamma_j\subseteq T$ to denote the indices in the $j$-th batch, $\Gamma_j[1]$ and $\Gamma_j[2\hspace*{-0.04in}:]$ to denote the first iteration and remaining iterations  in batch $j$, respectively. Note that $w^*_t = 0, t\in\Gamma_{2j-1}$ and $w_t^*=\alpha, t\in\Gamma_{2j}$. 

For $t\in\Gamma_1[2\hspace*{-0.04in}:]$ or $t=\Gamma_2[1]$,  by induction we can show that $w_{t} =\Pi_{\Omega}[w_{t-1} -\eta g'_1(w_{t-1})]=0$. Therefore, $w_t=w_t^*, t\in\Gamma_1[2\hspace*{-0.04in}:]$ and $w_{\Gamma_2[1]}=0$.  For $t\in\Gamma_2[2\hspace*{-0.04in}:]$ or $t=\Gamma_3[1]$, following the OGD update  $w_{t} =\Pi_{\Omega}[w_{t-1} - \eta g'_2(w_{t-1})]=\Pi_{\Omega}[w_{t-1} - 2\eta(w_{t-1}-\alpha)]=\alpha$. Therefore, $w_t=\alpha, t\in\Gamma_2[2\hspace*{-0.04in}:]$ and $t=\Gamma_3[1]$.   Following the same analysis, we have $w_t = w_t^*$ for $t\in\Gamma_j[2\hspace*{-0.04in}:]$, $w_{\Gamma_{2j-1}[1]}=\alpha$ and $w_{\Gamma_{2j}[1]}=0$. It means that the difference between the decision vector $w_t$ and the optimal solutions $w^*_t$ only happens at the first iterations of all batches. As a result, the dynamic regret is 
\begin{align*}
\sum_{j=1}^{m}&\max(g_1(\alpha) - g_1(0) , g_2(0) - g_2(\alpha)) = \sum_{j=1}^{m}\alpha^2\leq 2V^p_T 
\end{align*}
It is notable the key ingredient to achieve an $O(V^p_T)$ dynamic regret is to use a constant step size.   Next, we generalize this result to a broad family of loss functions.   In particular, we assume the sequence of  loss functions satisfy the following assumption.
\begin{ass}\label{ass:5}
Assume that every loss function $f_t(\cdot)$ is defined over $\R^d$ and is convex and smooth, i.e.,  for any $\w, \w' \in \R^d$, we have
\[
\|\nabla f_t(\w) - \nabla f_t(\w')\|_2 \leq L\|\w - \w'\|_2,
\]
where $L > 0$ is the smoothness constant. In addition, we assume that there exists $\w_t^*\in\Omega_t^*$ such that $\nabla f_t(\w_t^*)=0$.
\end{ass}
The condition  $\nabla f_t(\w_t^*)=0$ is referred to as the vanishing gradient condition. 
The examples considered before indeed satisfy {\bf Assumption~\ref{ass:5}}. Consider the policy of OGD:
\begin{equation}\label{eqn:ogd}
\begin{aligned}
\hspace*{-0.2in}\pi:\: \w_t \hspace*{-0.02in}= \hspace*{-0.02in}\left\{ \begin{array}{ll}\w_1\in\Omega&t=1\\  \Pi_{\Omega}[\w_{t-1} - \eta \nabla f_{t-1}(\w_{t-1})]&t>1\end{array}\right.
\end{aligned}
\end{equation}
 The following theorem states the dynamic regret bound of OGD with a constant step size.
\begin{thm}{(upper bound)}\label{thm:6}
Suppose {\bf Assumption~\ref{ass:5}} hold. By the policy $\pi$ in~(\ref{eqn:ogd}) with $\eta = 1/(2L)$, for any $\{f_1,\ldots, f_T\}\in\V_p$ we have
\[
\sum_{t=1}^T f_t(\w_t) - f_t(\w_t^*) \leq 2L\left(r^2 + 2rB_T\right).
\]
\end{thm}
\vspace*{-0.1in}
To prove the theorem, we first give the following lemma whose proof is deferred to Appendix. 
\begin{lemma}~\label{lem:1}
 Let $\w_t = \Pi_{\Omega}[\w_{t-1} - \eta\g_{t-1}], t>1$. Then 
\begin{align*}
&\g_t^{\top}(\w_t - \w_t^*)\leq\frac{\eta}{2}\|\g_t\|_2^2 +\frac{r\|\w_{t}^* - \w_{t+1}^*\|_2}{\eta}\\
&+ \frac{\|\w_t - \w_t^*\|_2^2 - \|\w_{t+1} - \w_{t}^*\|_2^2-\|\w_t^* - \w_{t+1}^*\|_2^2}{2\eta}
\end{align*}
\end{lemma}
\begin{proof}[Proof of Theorem~\ref{thm:6}]
Following Lemma~\ref{lem:1} and the convexity of $f_t(\w)$,  we have 
\begin{align}\label{eqn:fe}
&f_t(\w_t) - f_t(\w_t^*) \leq \frac{\eta}{2}\|\nabla f_t(\w_t)\|_2^2+ \frac{r\|\w_t^* - \w_{t+1}^*\|_2}{\eta}\\
&+ \frac{\|\w_t - \w_t^*\|_2^2 - \|\w_{t+1} - \w_{t+1}^*\|_2^2 - \|\w^*_{t+1} - \w^*_t\|_2^2}{2\eta}\nonumber
\end{align}
By the smoothness of $f(\w)$, for any $\w\in\R^d$
\[
f_t(\w) - f_t(\w_t) \leq \langle \nabla f_t(\w_t), \w - \w_t\rangle + \frac{L}{2}\|\w - \w_t\|_2^2
\]
Let $\w=\w'_t = \w_t - \frac{1}{L}\nabla f_t(\w_t)$ in the above inequality, we have
$f_t(\w'_t) - f_t(\w_t)\leq  - \frac{\|\nabla f_t(\w_t)\|_2^2}{2L}$. 
By convexity of $f_t(\w)$, 
\[
f_t(\w'_t)\geq f_t(\w_t^*) + \nabla f_t(\w_t^*)^{\top}(\w'_t - \w_t^*)=f_t(\w_t^*)
\]
where the equality follows the vanishing gradient condition.  Then
\[
f_t(\w^*_t) - f_t(\w_t)\leq f_t(\w'_t) - f_t(\w_t) \leq - \frac{\|\nabla f_t(\w_t)\|_2^2}{2L}
\]
Combing the inequality above with~(\ref{eqn:fe}), we have
\begin{align*}
& f_t(\w_t) - f_t(\w_t^*) \leq \eta L(f_t(\w_t) - f_t(\w_t^*))\\
&+ \frac{\|\w_t - \w_t^*\|_2^2 - \|\w_{t+1} - \w_{t+1}^*\|_2^2}{2\eta} + \frac{r\|\w_t^* - \w_{t+1}^*\|_2}{\eta}
\end{align*}
By summing over $t=1,\ldots, T$, we have
\begin{align*}
\sum_{t=1}^T&(f_t(\w_t) - f_t(\w_t^*))\leq \frac{1}{1-\eta L}\left(\frac{r^2}{2\eta} + \frac{r}{\eta}V^p_T\right)
\end{align*}
We complete the proof by choosing $\eta = 1/[2L]$.
\end{proof}
{\bf Remark: } From Theorem~\ref{thm:6}, we can see that OGD can achieve an $O(\max(V^p_T, 1))$ dynamic regret  for a sequence of loss functions in $\V_p$ that satisfy {\bf Assumption~\ref{ass:5}} with only the gradient feedback, which is comparable to that achieved in the full information feedback. The instance 1 and 2 in Section~\ref{sec:exam} has $V^p_T = O(1)$ and $V^p_T\approx 4C\sqrt{T}$, respectively. Therefore,  using OGD with $\eta=C$ we can obtain an $O(1)$ and $O(\sqrt{T})$ dynamic regret.

Finally, it is worth mentioning that the OGD with restarting proposed in Besbes et al.' work achieves an $O(T^{1/3})$ dynamic  regret for instance 1 and an $O(T^{5/6})$ dynamic regret for instance 2 due to that the functional variation for the first instance is bounded by a constant and for the second instance is bounded by $O(\sqrt{T})$. 


\section{Optimal Dynamic Regret with Noisy Gradient}
In this section, we focus on noisy gradient feedback, i.e., $\phi_t(\w_t, f_t)$ is only a noisy (sub)gradient 
of $f_t(\w)$ at $\w_t$. 

\subsection{A Lower Bound with Noisy Gradient Feedback}
Before presenting the upper bounds of the dynamic regret with noisy gradient feedback, we will first present a lower bound. For establishing the  lower bound, we consider the following class of  policies
\begin{align}\label{eqn:lowerpi}
\pi: \; \w_t = \left\{\begin{array}{ll}\pi_1(U), & t=1\\ \pi_t(\{\phi_{\tau}(\w_{\tau}, f_{\tau})\}_{\tau=1}^{t-1}, U), & t\geq 1\end{array}\right.
\end{align}
where $U, \pi_1, \pi_t$ are defined similarly as before, and $\phi_t(\w_t, f_t)$ is a noisy subgradient of $f_t$ at $\w_t$. 
In particular, we assume the noisy gradient is given by $\phi_t(\w_t, f_t)\in\partial f_t(\w_t) + \epsilon_t$  with $\epsilon_t$ satisfying the following condition. 
\begin{ass}\label{ass:noisyg}
$\epsilon_t\in\R^d, t\geq 1$ are iid random vectors with zero mean and covariance matrix $\Sigma$ with bounded entries such that $tr(\Sigma)\leq \lambda^2$. Let $P(\cdot)$ denote the cumulative function of $\epsilon_t$. There exists a constant $\tilde C$ such that for any $a\in\R^d$, $\int \log\left(\frac{d P(\y)}{d P(\y + \a)}\right)d P(\y)\leq \tilde C\|\a\|_2^2$. 
\end{ass} 
When the noise vectors are independent Gaussian random vectors with zero mean and covariance matrices with entries uniformly bounded by $\sigma^2$, the above assumption is satisfied with $\tilde C = 1/(2\sigma^2)$ and $\lambda^2=d\sigma^2$~\cite{besbes-2013-optimization}.  

\begin{thm}(lower bound)\label{thm:lower}   
For any $1\leq B_T\leq T$ and $\kappa\in(1/2,1)$, there exist $\V_p'\subset\V_p$ and  $C(\kappa)>0$ independent of $T$ and $B_T$ such that for any policy $\pi$ in~(\ref{eqn:lowerpi}) under the noisy gradient feedback that satisfies {\bf Assumption~\ref{ass:noisyg}}, we have 
\[
\Rt_\phi^{\pi}(\V'_p, T) \geq C(\kappa)B_T^{\kappa} T^{1 - \kappa}.
\]
\end{thm}
{\bf Remark: } 
Note that from the proof presented below when $\kappa\rightarrow 1/2$, $C(\kappa)\rightarrow 0$. However, the above lower bound can be used to argue that it is impossible to achieve a better dynamic regret than $O(B_T^{1/2}T^{1/2})$ with the noisy gradient feedback for any sequence of loss functions. We prove this by contradiction. In particular, assume there exists an algorithm under the noisy gradient feedback achieves better bound than $O(B_T^{1/2}T^{1/2})$ for any sequence of loss functions. We can consider two lower orders $O(B_T^{\alpha}T^{1-\alpha})$ with $1> \alpha>1/2$ and $O(B_T^{\alpha}T^{\beta})$  with $\alpha\leq 1/2, \beta\leq 1/2$ and $\alpha+\beta<1$. First, we assume $O(B_T^{\alpha}T^{1-\alpha})$ is achievable. By Theorem~\ref{thm:lower} we know that there exists $\alpha<\kappa<1$ (e.g., $\kappa=\frac{\alpha+1}{2}$) and $\V_p'$ such that $\Rt_\phi^{\pi}(\V'_p, T) \geq \Omega(B_T^{\kappa} T^{1 - \kappa})\geq\Omega(B_T^{\alpha}T^{1-\alpha})$, which yields a contradiction. To show that the second lower bound is unachievable, we can construct a $B_T$ such that $B_T^{\kappa} T^{1 - \kappa}\geq\Omega(B_T^{\alpha}T^{\beta}) $, i.e., $B_T\geq \Omega(T^{\frac{\beta+\kappa-1}{\kappa-\alpha}})$, where $\beta+\kappa-1< \kappa-\alpha$. 
\begin{proof}
We construct two functions over the domain $\Omega = [-1/2, +1/2]$. They are
\begin{eqnarray}
f(x) = \left\{
\begin{array}{cc}
\frac{1}{1+\gamma}\delta^{1+\gamma}- \delta^{\gamma} x & x \in [-1/2, 0] \\
\frac{1}{1+\gamma}|x - \delta|^{1+\gamma} & x \in [0, 2\delta] \\
-\frac{1+2\gamma}{1+\gamma}\delta^{1+\gamma} + \delta^{\gamma} x & x \in [2\delta, 1/2]
\end{array}
\right., \\\quad
g(x) = \left\{
\begin{array}{cc}
-\frac{1+2\gamma}{1+\gamma}\delta^{1+\gamma}- \delta^{\gamma} x & x \in [-1/2, -2\delta] \\
\frac{1}{1+\gamma}|x + \delta|^{1+\gamma} & x \in [-2\delta, 0] \\
\frac{1}{1+\gamma}\delta^{1+\gamma} + \delta^{\gamma} x & x \in [0, 1/2]
\end{array}
\right. \label{eqn:loss}
\end{eqnarray}
where $0<\delta <1/2$ and $\gamma >0$ will be determined later. It is easy to verify that both $f(\cdot)$ and $g(\cdot)$ are convex but not strongly convex. It is also easy to see that the optimal solutions  for $f(\cdot)$ and $g(\cdot)$ are $x^*_f = \delta$ and $x^*_g = -\delta$, respectively. Hence $|x_f^* - x_g^*| = 2\delta$. For a given budget $B_T$, we will construct a subset of $\V_p$ by only considering the sequence of these two loss functions. For some $\Delta_T\in\{1,\ldots, T\}$ that, we divide the entire sequence $T$ into $m=\lceil T/\Delta_T\rceil$ batches, denoted by $\T_1, \ldots, \T_m$, each with size of $\Delta_T$ (except perhaps $\T_m$), i.e. $\T_j = \{(j - 1)\Delta_T+1, \ldots, \min(j\Delta_t, T)\}, j=1, \ldots, m$. To generate the sequence of loss functions $f_1,\ldots, f_T$, at the beginning of each batch $\T_j$, we randomly choose between the two functions $f(\cdot)$ and $g(\cdot)$ and the same loss function will be used throughout the batch. We denote by $\V'_p=\{\{f_t(\cdot), t=1, \ldots, T\}\}$ the set of a sequence of randomly sampled loss functions, and by $X_1, \ldots, X_T$ the sequence of solutions generated by any policy in~(\ref{eqn:lowerpi}). Let $\delta=B_T\Delta_T/2T$. For any $f \in \V'_p$, we have
\[
V^p_T=\sum_{j=2}^m |x^*_j - x^*_{j-1}| \leq\left(\lceil T/\Delta_T\rceil - 1\right)2\delta \leq \frac{T2\delta}{\Delta_T}=B_T
\]
Therefore, $\V'_p\subset \V_p$. We denote by $\P^\pi_f$  the  probability measure under policy $\pi$ when $f$ is the sequence of the loss functions, and by $\E^\pi_f$ the associated expectation operator. 
Set $\Delta_T=\max\left\{(\frac{4^\gamma}{4\tilde C})^{1/(2\gamma+1)}(\frac{T}{B_T})^{2\gamma/(2\gamma+1)}, 1\right\}$. Then 
\begin{align*}
\tilde C&\E^{\pi}_f\left[\sum_{t \in \T_j} (\nabla f(X_t) - \nabla g(X_t))^2 \right] \leq \tilde C\sum_{t\in\T_j}4\delta^{2\gamma}\\
&\leq  4\tilde C\Delta_T\delta^{2\gamma} \leq 4\tilde C\frac{B_T^{2\gamma}\Delta^{2\gamma+1}}{2^{2\gamma}T^{2\gamma}}\leq \max(1, \frac{4\tilde C B_T^{2\gamma}}{T^{2\gamma}4^\gamma})\\
&\leq \max(1, \frac{4\tilde C}{4^\gamma})\leq \max(1, 4\tilde C)\triangleq\beta
\end{align*}
where we use the condition that $B_T\leq T$.  Using Lemma A-1 and A-2 from~\citep{besbes-2013-optimization},  we have
\[
\max\left\{\P^\pi_{f}\{X_t > 0\}, \P^\pi_{g}\{X_t \leq 0\} \right\} \geq \frac{1}{4e^{\beta}}, \forall t
\]
Using the above result, we have
\begin{align*}
&\Rt^\pi_\phi(\V_p', T) \geq \E\left[\sum_{j=1}^{m} \sum_{t \in \T_j} f_t(X_t) - f_t(x_j^*) \right] \\
& =  \frac{1}{2}\sum_{j=1}^m\E^\pi_f\left[ \sum_{t \in \T_j} f(X_t) - f(x_f^*)\right] \\
&+\frac{1}{2}\sum_{j=1}^m \E^\pi_g\left[ \sum_{t \in \T_j}g(X_t) - g(x_g^*)\right] \\
& \geq \frac{1}{2}\sum_{j=1}^{m} \sum_{t \in \T_j} \P\{X_t \leq  0\}\left(f(0) - f(x_f^*) \right) \\
&+ \frac{1}{2}\sum_{j=1}^{m} \sum_{t \in \T_j}\P\{X_t > 0\} \left(g(0) - g(x_g^*)\right)\\
& =  \frac{\delta^{1+\gamma}}{2(1+\gamma)}\sum_{j=1}^{m}\sum_{t \in \T_j} \left(\P\{X_t \leq  0\} + \P\{X_t > 0\}\right)\\
& \geq \frac{\delta^{1+\gamma}T}{8(1+\gamma)e^{\beta}}.
\end{align*}
In the above derivations, the first expectation is taking over all randomness in $\pi$ and $f_1,\ldots, f_T$, the second inequality holds because 
\begin{align*}
&\E[f(X_t)]=\E[f(X_t)|X_t>0]\Pr(X_t>0)\\
&+\E[f(X_t)|X_t\leq 0]\Pr(X_t\leq 0)\geq f(0)\Pr(X_t\leq 0)\\
&\E[g(X_t)]=\E[g(X_t)|X_t>0]\Pr(X_t>0)\\
&+\E[g(X_t)|X_t\leq 0]\Pr(X_t\leq 0)\geq g(0)\Pr(X_t> 0)
\end{align*}
where the inequalities are due to $f(x)\geq 0, g(x)\geq 0$ and $f(x)\geq f(0)$ when $x\leq 0$ and $g(x)\geq g(0)$ when $x\geq 0$. 
To proceed, we plug in the value of $\delta$ into the lower bound of $\Rt^\pi_\phi(\V_p', T)$
\begin{align*}
&\Rt^\pi_\phi(\V_p', T)\geq \frac{T}{8e^{\beta}(1+\gamma)}\frac{B_T^{1+\gamma}\Delta_T^{1+\gamma}}{2^{1+\gamma}T^{1+\gamma}}\\
&\geq \frac{4^{\gamma(1+\gamma)/(1+2\gamma)}}{8e^{\beta}(1+\gamma)2^{1+\gamma}(4\tilde C)^{(1+\gamma)/(1+2\gamma)}}\frac{B_T^{1+\gamma - (1+\gamma)\frac{2\gamma}{1+2\gamma}}}{T^{\gamma -  (1+\gamma)\frac{2\gamma}{1+2\gamma}}}\\
&= \frac{4^{\gamma(1+\gamma)/(1+2\gamma)}}{8e^{\beta}(1+\gamma)2^{1+\gamma}(4\tilde C)^{(1+\gamma)/(1+2\gamma)}}B_T^{\frac{1+\gamma}{1+2\gamma}}T^{1-\frac{1+\gamma}{1+2\gamma}}
\end{align*}
Let $\gamma = (1-\kappa)/(2\kappa-1)$, i.e., $\kappa=(1+\gamma)/(1+2\gamma)$. Then
\begin{align*}
\Rt^\pi_\phi(\V_p', T)\geq\frac{(4^\gamma)^{\kappa}}{(1+\gamma)2^\gamma}\frac{1}{16e^{\beta}(4\tilde C)^{\kappa}}B_T^{\kappa}T^{1-\kappa}.
\end{align*}
\end{proof}
In the next two subsections, we consider two types of noisy gradient feedback, namely a stochastic subgradient feedback that is an unbiased estimation of the true subgradient and a bandit feedback that gives an unbiased  estimation of the subgradient of a smoothed function instead of the original function. We show that under the two noisy gradient feedback, we are able to achieve an optimal dynamic regret of  $O(\sqrt{V^p_TT})$. Furthermore, for smooth loss functions under the two-point bandit feedback, we establish an even better upper bound by leveraging the gradient variation in the form of $O(\max(\sqrt{V^p_TV^g_T}, V^p_T))$, which when the gradient variation is small matches the lower bound presented in Proposition~\ref{prop:2}.

\subsection{Online Learning with Bounded Stochastic Gradient Feedback}\label{sec:general}
We adopt the policy defined by OGD using the noisy gradient feedback, i.e., 
\begin{equation}\label{eqn:sgd}
\begin{aligned}
\pi:\: \w_t= \left\{ \begin{array}{ll}\w_1\in\Omega&t=1\\  \Pi_{\Omega}[\w_{t-1} - \eta \phi_{t-1}(\w_{t-1}, f_{t-1})]&t>1\end{array}\right.
\end{aligned}
\end{equation}
where $\phi_t(\w_t, f_t)\in\partial f_t(\w_t) + \epsilon_t$ is a noisy subgradient with $\epsilon_t$ satisfying {\bf Assumption~\ref{ass:noisyg}}.
The upper bound of the dynamic regret  of OGD with an appropriate step size is presented below. 
\begin{thm}(upper bound)\label{thm:3}
Suppose {\bf Assumption~\ref{ass:noisyg}} hold. Assume $\|\partial f_t(\w)\|_2\leq G$, for any $\w\in\Omega$ and $1\leq t\leq T$.
By the policy $\pi$ in~(\ref{eqn:sgd}) with $
\eta = \sqrt{\frac{r^2 + 2r B_T}{T(G^2 + \lambda^2)}}$,
we have
\[
\Rt_\phi^{\pi}(\V_p, T) \leq \sqrt{(r^2 + 2rB_T)(G^2 + \lambda^2)T}.
\]
\end{thm}
\begin{proof} Let $\E_t[\cdot]$ denote the expectation over the randomness in $\phi_t$ given the randomness before $t$. We abuse the notation $\w_{T+1}^*=\w_T^*$. Note that $$\E_t\left[\|\phi_t(\w_t, f_t)\|_2^2\right]\leq (G^2+\lambda^2).$$ 
Following Lemma~\ref{lem:1} and the convexity of $f_t(\w)$, we have
\begin{align*}
&\E_t\left[f_t(\w_t) - f_t(\w_t^*)\right] \leq \E_t\left[\langle\phi_t(\w_t,  f_t), \w_t - \w_t^* \rangle\right]  \\
& \leq  \frac{\|\w_t - \w_t^*\|_2^2}{2\eta} - \frac{\|\w_{t+1} - \w_{t+1}^*\|_2^2}{2\eta} - \frac{\|\w_t^* - \w_{t+1}^*\|_2^2}{2\eta} \\
&+ \frac{\eta}{2}(G^2 + \lambda^2) + \frac{r}{\eta}\| \w_t^* - \w_{t+1}^*\|_2 
\end{align*}
Hence, by summing the above inequalities over $t=1,\ldots, T$ we have
\begin{align*}
\E&\left[\sum_{t=1}^T f_t(\w_t) - f_t(\w_t^*)\right] \leq \frac{1}{2\eta}\left(r^2 + 2rV^p_T \right) + \frac{\eta}{2}G^2_\lambda T
\end{align*}
where $G_\lambda^2=G^2+\lambda^2$. 
Since the above inequality holds for any $\w_t^*\in\Omega_t^*$, we thus conclude 
\begin{align*}
\E&\left[\sum_{t=1}^T f_t(\w_t) - f_t(\w_t^*)\right]\leq \frac{1}{2\eta}(r^2 + 2rB_T)+ \frac{\eta}{2}G_\lambda^2T
\end{align*}
We complete the proof by choosing $\eta = \sqrt{\frac{r^2 + 2r B_T}{T(G^2 + \lambda^2)}}$.
\end{proof}
{\bf Remark:} From Theorem~\ref{thm:3}, we can see that OGD can achieve an $O(\sqrt{\max(V^p_T, 1)T})$ dynamic regret with a step size $\eta=C\sqrt{\max(V^p_T,1)/T}$.  Compared to OGD with restarting proposed in~\citep{besbes-2013-optimization}, our result could be better when $\sqrt{V^p_TT}\leq O((V^f_T)^{1/3}T^{2/3})$, i.e., $V^p_T\leq O((V^f_T)^{2/3}T^{1/3})$. 

\subsection{Online Learning with Bandit Feedback}
In this subsection, we analyze the dynamic regret with bandit feedback by building on previous work. Bandit feedback has been analyzed before for the static regret. In particular, using one-point bandit feedback \citet{DBLP:conf/soda/FlaxmanKM05} showed an $O(T^{3/4})$ static regret bound, while \citet{DBLP:conf/colt/AgarwalDX10} established an optimal static regret bound  of  $O(\sqrt{T})$ using two-point bandit feedback. Recently, \citet{DBLP:conf/colt/ChiangLL13} derived a  variational static regret bound in the two-point bandit setting that depends on $\sqrt{V^g_T}$ where $V^g_T$ is the gradient variation defined in Section~\ref{sec:intro}. In order to have optimal dynamic regret bounds, we also consider two-point bandit setting  and show that the previous algorithms in~\cite{DBLP:conf/colt/AgarwalDX10,DBLP:conf/colt/ChiangLL13} by adjusting the step size can achieve an $O(\sqrt{V^p_TT})$ dynamic regret for general Lipschitz continuous loss functions and an $O(\max(\sqrt{V^g_TV^p_T}, V^p_T))$ dynamic regret for smooth loss functions. Below, we present more details. The omitted proof can be found in Appendix. 

Similar to previous work, we assume that $f_t(\w)$ is $G$-Lipschitz continuous and  $R_1\mathbb B\subseteq \Omega\subseteq R_2\mathbb B$ where $\mathbb B=\{\w\in\R^d: \|\w\|_2\leq 1\}$ is the unit ball centered at $0$. Let $\u_t\in\R^d$ be a random unit vector,  $\mathbf e_i\in\R^d$ be the $i$-th canonical vector, $\w_1=0$. For Lipschitz continuous loss functions, the update is given by 
\begin{align}\label{eqn:ogdbandit}
\w_{t+1} = \Pi_{(1-\xi)\Omega}[\w_t - \hat\g_t]
\end{align}
where $\xi\in(0,1)$ and $\hat\g_t$ is computed from two-point bandit feedback
\[
\hat\g_t = \frac{d}{2\delta}[f_t(\w_t + \delta\u_t) - f_t(\w_t - \delta\u_t)]\u_t
\]
with $\delta=\xi R_1$. For any $\w_t\in(1-\xi)\Omega$ and any unit vector $\u$, $\w_t + \delta\u\in\Omega$~\cite{DBLP:conf/soda/FlaxmanKM05}.  It can be shown that $\hat\g_t$ is an unbiased stochastic gradient of the function $\hat f_t(\w) = \E_\u[f_t(\w + \delta\u)]$. Importantly, $\|\hat\g_t\|_2\leq Gd$.  The following theorem states the dynamic regret bound for the policy in~(\ref{eqn:ogdbandit}). 
\begin{thm}\label{thm:8}
Assume $f_t(\w)$ is $G$-Lipschitz continuous. By the policy in~(\ref{eqn:ogdbandit}) with $\xi=\frac{1}{T}$, $\delta=\xi R_1$,  and  $\eta = \sqrt{\frac{r^2 + 2r B_T}{TG^2d^2}}$, we have
\begin{align*}
\E&\left[\sum_{t=1}^T\frac{1}{2}(f_t(\wh^1_t)+f(\wh^2_t))\right] - \sum_{t=1}^Tf_t(\w_t^*)\\
&\leq \sqrt{(r^2+2rB_T)G^2d^2T} + G(3R_1+R_2).
\end{align*}
where $\wh^1_t=\w_t + \delta \u_t, \wh^2_t=\w_t - \delta\u_t$. 
\end{thm}
{\bf Remark:}  The dynamic regret averaged over two decisions  with two-point bandit feedback is in  the same order of $\sqrt{\max(V^p_T,1)T}$ to that in Theorem~\ref{thm:3} with  stochastic gradient feedback.

Finally, we present an upper bound for smooth loss functions by leveraging the gradient variation, which leads to an improved  dynamic regret bound compared to Lipschitz continuous loss functions. The updates are based on the META algorithm proposed in~\citep{DBLP:conf/colt/ChiangLL13}, which is presented in Algorithm~\ref{alg:1}. It was proved to  achieve a  better static regret of $O(\sqrt{V^g_T})$ than $O(\sqrt{T})$. 
Below, we show that the same policy but with a different step size can achieve an improved dynamic regret, i.e., $O(\max(\sqrt{V^g_T\max(V^p_T,1)}, V^p_T))$ for a sequence of  smooth loss functions from the following set 
\[
\V_{p,g} = \{\{f_1,\ldots, f_T\}: V^p_T\leq B_T, V^g_T\leq S_T\}
\]
\begin{algorithm} [t]
    \caption{META algorithm}
    \label{alg:1}
    \begin{algorithmic}[1]
    \STATE Initialize solution $\w_1 = \wh_1=0$ and $\gh_0=0$
    \FOR{$t = 1, \ldots, T$}
        \STATE  Choose $i_t$ uniformly from $[d]$
        \STATE  Submit  $\wh^1_t=\wh_t + \delta \mathbf e_{i_t}$ and $\wh^2_t=\wh_t -  \delta \mathbf e_{i_t}$
        \STATE Receive the feedback $f_t(\wh^1_t)$ and $f_t(\wh^2_t)$ and let $v_{t,i_t}=\frac{1}{2\delta}(f_t(\wh^1_t) - f(\wh^2_t))$
        \STATE Compute
        \begin{align*}
        \g_t & = d(v_{t,i_t} - \gh_{t-1, i_t})\mathbf e_{i_t}, \text{ and } \\
         \gh_t &= d(v_{t,i_t} - \gh_{t-1, i_t})\mathbf e_{i_t} + \gh_{t-1}
        \end{align*}
        \STATE Update
        \begin{align*}
        \w_{t+1} &=\Pi_{(1-\xi)\Omega}[\w_t - \eta \g_t]\\
         \wh_{t+1} &=\Pi_{(1-\xi)\Omega}[\w_{t+1} - \eta \gh_t]
         \end{align*}
    \ENDFOR
    \end{algorithmic}
\end{algorithm}
The theorem below states the result. 
\begin{thm}\label{thm:last}
Assume  $\{f_1,\ldots, f_T\}\in \V_{p,g}$ and $f_t(\w)$ is $L$-smooth for any $t\geq 1$. By the policy in Algorithm~\ref{alg:1} with $\xi=\frac{1}{T}$, $\delta=\xi R_1$ and  $ \eta = \min\left(\sqrt{\frac{(2rB_T + r^2)}{8S_Td^4}}, \frac{1}{4Ld^{3/2}\sqrt{\ln T}} \right)$, we have
\begin{align*}
\E&\left[\sum_{t=1}^T\frac{1}{2}(f_t(\wh^1_t)+f(\wh^2_t))\right] - \sum_{t=1}^Tf_t(\w_t^*)\\
&\leq O\left(\max\left\{d^2\sqrt{S_T\max(B_T,1)}, d^{3/2} \max(B_T,1)\right\}\right).
\end{align*}
where $\wh^1_t=\w_t + \delta \u_t, \wh^2_t=\w_t - \delta\u_t$. 
\end{thm}
{\bf Remark:}  When the gradient variation is small such that the upper bound is dominated by $O(V^p_T)$, it matches the lower  bound established in Proposition~\ref{prop:2}. Finally, we note that a similar upper bound can be achieved for linear loss functions by extending the static regret analysis in~\cite{DBLP:conf/colt/ChiangLL13} to the dynamic regret similarly to the proof of Theorem~\ref{thm:last}.

\section{Conclusions}\label{sec:discuss}
In this paper, we have considered dynamic regret for online learning under true and noisy gradient feedback. We have developed several lower and upper bounds of the dynamic regret  based on the path variation that measures the temporal changes in the optimal solutions. In light of the presented lower bounds, the achieved upper bounds are optimal for non-strongly convex loss functions when the clairvoyant moves slowly. An interesting  question that remains open is that what is the optimal dynamic regret bound for strongly convex loss functions in terms of the path variation. 

\section*{Acknowledgements}
The authors would like to thank the anonymous reviewers for their helpful  comments. T. Yang was supported in part by NSF (1463988, 1545995). 

\bibliography{all}
\bibliographystyle{icml2015}

\appendix
\section{Proof of Lemma~\ref{lem:1}}
Let $\w_t' = \w_t  - \eta \g_t$. Thus $\w_{t+1} = \Pi_{\Omega}[\w_t']$. 
\begin{align*}
&\frac{1}{2}\|\w_{t+1} - \w_t^*\|_2^2\leq \frac{1}{2}\|\w_t' - \w_t^*\|_2^2 = \frac{1}{2}\|\w_t - \eta \g_t  - \w_t^*\|_2^2 \\
&=\frac{1}{2}\|\w_t - \w_t^*\|_2^2 - \eta \g_t^{\top}(\w_t - \w_t^*) + \frac{1}{2}\eta^2\|\g_t\|_2^2
\end{align*}
Then 
\begin{align*}
&\g_t^{\top}(\w_t - \w_t^*)\leq \frac{1}{2}\|\w_t - \w_t^*\|_2^2 - \frac{1}{2}\|\w_{t+1} - \w_t^*\|_2^2 \\
& + \frac{1}{2}\eta^2\|\g_t\|_2^2\\
&=\frac{1}{2\eta}\|\w_t - \w_t^*\|_2^2 - \frac{1}{2\eta}\|\w_{t+1} -\w_{t+1}^* + \w_{t+1}^* - \w_t^*\|_2^2 \\
&  + \frac{1}{2}\eta\|\g_t\|_2^2\\
&=\frac{1}{2\eta}\|\w_t - \w_t^*\|_2^2 - \frac{1}{2\eta}\|\w_{t+1} -\w_{t+1}^*\|_2^2\\
& - \frac{1}{2\eta} \|\w_{t+1}^* - \w_t^*\|_2^2 + \frac{1}{\eta}(\w_{t+1}^* -\w_{t+1})^{\top}(\w_{t}^* - \w_{t+1}^*)\\
&   + \frac{1}{2}\eta\|\g_t\|_2^2\\
&\leq\frac{1}{2\eta}\|\w_t - \w_t^*\|_2^2 - \frac{1}{2\eta}\|\w_{t+1} -\w_{t+1}^*\|_2^2\\
& - \frac{1}{2\eta} \|\w_{t+1}^* - \w_t^*\|_2^2 + \frac{1}{\eta}r\|\w_{t}^* - \w_{t+1}^*\|_2   + \frac{1}{2}\eta\|\g_t\|_2^2
\end{align*}

\section{Proof of Theorem~\ref{thm:8}}
Define $\hat f_t(\w)$ as 
\begin{align*}
\hat f_t(\w) = \E_{\u}[f_t(\w + \delta \u)]
\end{align*}
where $\u$ is a random unit vector. 
We first give the following lemma. 
\begin{lemma}\label{lem:10}
Let $\wh_t^* = (1-\xi)\w_t^*$. 
\begin{align*}
&\sum_{t=1}^T\frac{1}{2}(f_t(\wh_t^1)+f_t(\wh_t^2)) - \sum_{t=1}^Tf_t(\w_t^*)\\
&\leq \sum_{t=1}^T \hat f_t(\w_t) - \sum_{t=1}^T\hat f_t(\wh_t^*) + 3TG\delta + TGR_2\xi
\end{align*}
\end{lemma}
Note that the updating rule is OGD with a noisy gradient applied to a sequence of functions $\hat f_t(\w), t=1,\ldots, T$. Following~\cite{DBLP:conf/colt/AgarwalDX10}, $\hat\g_t$ is bounded by $\|\hat \g_t\|_2\leq Gd$. Then following the proof of Theorem 6, we have,
\begin{align*}
 &\sum_{t=1}^T \hat f_t(\w_t) - \sum_{t=1}^T\hat f_t(\wh_t^*) \leq \frac{\|\w_1 - \wh_1^*\|_2^2}{2\eta} \\
 &+ \sum_{t=1}^T\frac{1}{\eta}\|\w_{t+1} -\wh_{t+1}^*\|_2\|\wh_t^* - \wh^*_{t+1}\|_2 + \frac{\eta}{2}G^2d^2T
\end{align*}
Since $\w_t,\wh_t^*\in(1-\xi)\Omega$, we have
\begin{align*}
\|\w_t - \wh_t^*\|_2\leq r
\end{align*}
due to Assumption~\ref{ass:1}.  In addition, 
\begin{align*}
&\|\wh^*_t - \wh^*_{t+1}\|_2 = \|(1-\xi)\w_t^* - (1-\xi)\w^*_{t+1}\|_2\\
&\leq (1-\xi)\|\w_t^* - \w^*_{t+1}\|_2\leq \|\w_t^* - \w^*_{t+1}\|_2
\end{align*}

Then
\begin{align*}
 &\sum_{t=1}^T \hat f_t(\w_t) - \sum_{t=1}^T\hat f_t(\wh_t^*) \leq \frac{1}{2\eta}(r^2 +2rV_T^p)   + \frac{\eta}{2}G^2d^2T\\&\leq  \frac{1}{2\eta}(r^2 +2rB_T)   + \frac{\eta}{2}G^2d^2T
\end{align*}
By plugging the value of  $\eta$, we have
\begin{align*}
 \sum_{t=1}^T \hat f_t(\w_t) - \sum_{t=1}^T\hat f_t(\wh_t^*) \leq \sqrt{(r^2 +2rB_T)G^2d^2T}
\end{align*}
Then combining the above inequality with Lemma 13, we have

\begin{align*}
&\sum_{t=1}^T\frac{1}{2}(f_t(\wh_t^1)+f_t(\wh_t^2)) - \sum_{t=1}^Tf_t(\w_t^*)\\
&\leq \sqrt{(r^2 +2rB_T)G^2d^2T} + 3TG\delta + TGR_2\xi\\
&\leq  \sqrt{(r^2 +2rB_T)G^2d^2T}+ G(3R_1 +R_2 ).
\end{align*}

\subsection{Proof of Lemma~\ref{lem:10}}
The proof is almost identical to that of Lemma 2 in~\cite{DBLP:conf/colt/AgarwalDX10}. 
By the Lipschitz property of $f_t(\w)$, we have
\begin{align*}
f_t(\wh_t^1) = f_t(\w_t + \delta\u_t)\leq f_t(\w_t)  + G\delta\|\u_t\|_2\\
f_t(\wh_t^2) = f_t(\w_t - \delta\u_t)\leq f_t(\w_t)  + G\delta\|\u_t\|_2
\end{align*}
Since $\|\u_t\|_2=1$, thus
\begin{align*}
\frac{1}{2}(f_t(\wh_t^1) + f_t(\wh_t^2))\leq f_t(\w_t) + G\delta
\end{align*}
By the Lipschitz property and $\Omega\subset R_2\B$, we have for any $\w\in\Omega$
\[
f_t((1-\xi)\w)\leq f_t(\w) + GR_2\xi
\]
Further for any $\w\in(1-\xi)\Omega$, 
\begin{align*}
|f_t(\w) - \hat f_t(\w)|& = |f_t(\w) - \E_tf_t(\w + \delta\u)|\\
&\leq \E_t|f_t(\w) - f_t(\w+ \delta\u)|\leq G\delta
\end{align*}
Then 
\[
f_t(\w_t)\leq \hat f_t(\w_t) + G\delta, \quad \text{ and } \hat f_t((1-\xi)\w_t^*)\leq f_t((1-\xi)\w_t^*) + G\delta
\]
Combining the above inequalities, we get 
\begin{align*}
&\frac{1}{2}(f_t(\wh_t^1) + f_t(\wh_t^2)) + \hat f_t((1-\xi)\w_t^*)\\
&\leq  f_t(\w_t) + G\delta + f_t((1-\xi)\w_t^*) + G\delta\\
&\leq f_t(\w_t) + f_t(\w_t^*) + GR_2\xi +  2G\delta\\
&\leq \hat f_t(\w_t) + f_t(\w_t^*)  + GR_2\xi + 3G\delta
\end{align*}
As a result, 
\begin{align*}
&\sum_{t=1}^T\frac{1}{2}(f_t(\wh_t^1)+f_t(\wh_t^2)) - \sum_{t=1}^Tf_t(\w_t^*)\\
&\leq \sum_{t=1}^T \hat f_t(\w_t) - \sum_{t=1}^T\hat f_t((1-\xi)\w_t^*) + 3TG\delta + TGR_2\xi
\end{align*}

\section{Proof of Theorem~\ref{thm:last}}
The proof follows similarly to the analysis in~\cite{DBLP:conf/colt/ChiangLL13}. We present a series of Lemmas with some of the Lemmas' proof omitted due to that they are identical to that in~\cite{DBLP:conf/colt/ChiangLL13}. To simply the presentation, we denote by $O(1)$ any constant independent of $T$. 
\begin{lemma}\label{lem:11}
\begin{align*}
&\sum_{i=1}^T\frac{1}{2}(f_t(\wh_t^1) + f_t(\wh^2_t)) - \sum_{t=1}^T f_t(\w_t^*)\\
&\leq \sum_{t=1}^Tf_t(\w_t) - \sum_{t=1}^Tf_t((1-\xi)\w_t^*) + O(1)
\end{align*}
\end{lemma}
The proof of this lemma is presented later. 

\begin{lemma}
Let $\wh_t^*=(1-\xi) \w_t^*$. 
\[
\sum_{t=1}^Tf_t(\w_t) - \sum_{t=1}^Tf_t(\wh_t^*)\leq \sum_{t=1}^T\nabla f_t(\w_t)^{\top}(\w_t - \wh_t^*)
\]
\end{lemma}
The lemma above follows the convexity of $f_t(\w)$. 
\begin{lemma}
Let $\wh_t^*=(1-\xi) \w_t^*$. 
\begin{align*}
&\E\left[\sum_{t=1}^T\nabla f_t(\w_t)^{\top}(\w_t - \wh_t^*)\right]\\
&\leq \E\left[\sum_{t=1}^T\g_t^{\top}(\w_t - \wh_t^*)\right] + O(1)
\end{align*}
\end{lemma}
The proof of  the above lemma follows the same to the proof of Lemma 5 in~\cite{DBLP:conf/colt/ChiangLL13}.
\begin{lemma}
Define
\begin{align*}
S_t & = \eta_t \|\g_t - \g_{t-1}\|_2^2\\
A_t & = \frac{1}{2\eta}\|\w_t - \wh_t^*\|_2^2 - \frac{1}{2\eta}\|\w_{t+1} - \wh_t^*\|_2^2\\
C_t & = \frac{1}{2}\|\w_{t+1} - \wh_t\|_2^2 +\frac{1}{2\eta}\|\w_t - \wh_t\|_2^2
\end{align*}
Then 
\[
\sum_{t=1}^T\g_t^{\top}(\w_t - \wh_t^*)\leq \sum_{t=1}^T S_t  + \sum_{t=1}^TA_t - \sum_{t=1}^TC_t
\]
\end{lemma}
The above lemma is a result of the Lemma 4 in~\cite{DBLP:conf/colt/ChiangLL13}. 
Combining the above lemmas, we have
\begin{thm}
\begin{align*}
&\E\left[\sum_{i=1}^T\frac{1}{2}(f_t(\wh_t^1) + f_t(\wh^2_t)) - \sum_{t=1}^T f_t(\w_t^*)\right]\\
&\leq \E\left[\sum_{t=1}^T S_t  + \sum_{t=1}^TA_t - \sum_{t=1}^TC_t\right] + O(1)
\end{align*}
\end{thm}
To proceed we bound the three summation terms in the R.H.S.. 
\begin{lemma}\label{lem:16}
\[
\sum_{t=1}^TC_t \geq \frac{1}{4\eta}\E\left[\sum_{t=1}^T\|\wh_t - \wh_{t+1}\|_2^2\right] - O(1)
\]
\end{lemma}
This is the same to the Lemma 11 in~\cite{DBLP:conf/colt/ChiangLL13}.
\begin{lemma}\label{lem:17}
\begin{align*}
&\sum_{t=1}^T A_t\leq \frac{\|\w_1 - \wh_1^*\|_2^2}{2\eta} \\
&+ \frac{1}{\eta}\sum_{t=1}^T\|\w_{t+1} - \wh_{t+1}^*\|_2 \|\wh^*_t -\wh_{t+1}^*\|_2\\
&\leq \frac{1}{2\eta}(r^2 + 2rV^p_T)
\end{align*}
\end{lemma}
The proof follows similarly to that of Lemma 1. 
\begin{lemma}\label{lem:18}
\[ 
\sum_{t=1}^TS_t \leq 4\eta d^4V^g_T + 4\eta d^3L^2\ln T\E\left[\sum_{t=1}^T\|\wh_t - \wh_{t+1}\|_2^2\right]  + O(1)
\]
where $L$ is the smoothness parameter. 
\end{lemma}
The lemma follows the Lemma 12 in~\cite{DBLP:conf/colt/ChiangLL13}.
Combining Lemma~\ref{lem:16}, \ref{lem:17} and Lemma~\ref{lem:18}, we have

\begin{align*}
 \E&\left[\sum_{t=1}^T S_t  + \sum_{t=1}^TA_t - \sum_{t=1}^TC_t\right] \leq 4\eta d^4V^g_T + \frac{1}{2\eta}(r^2 + 2rV^p_T) \\
 &+ 4\eta d^3L^2\ln T\E\left[\sum_{t=1}^T\|\wh_t - \wh_{t+1}\|_2^2\right] \\
& -  \frac{1}{4\eta}\E\left[\sum_{t=1}^T\|\wh_t - \wh_{t+1}\|_2^2\right] + O(1)
\end{align*}
Since $\eta\leq \frac{1}{4d^{3/2}L\sqrt{\ln T}} $, then 
\begin{align*}
 &\E\left[\sum_{t=1}^T S_t  + \sum_{t=1}^TA_t - \sum_{t=1}^TC_t\right]\\
 & \leq 4\eta d^4V^g_T + \frac{1}{2\eta}(r^2 + 2rV^p_T)  + O(1)\\
 & \leq 4\eta d^4S_T + \frac{1}{2\eta}(r^2 + 2rB_T)  + O(1)
\end{align*}
Then by the value of $\eta$, we have
\begin{align*}
 &\E\left[\sum_{t=1}^T S_t  + \sum_{t=1}^TA_t - \sum_{t=1}^TC_t\right]\\
 & \leq 4\eta d^4S_T + \frac{1}{2\eta}(r^2 + 2rB_T)  + O(1)\\
 &\leq \max\left\{2\sqrt{2}\sqrt{d^4S_T(r^2 + 2rB_T)}, 4d^{3/2}L\sqrt{\ln T}(r^2 + 2rB_T)\right\} \\
 &+ O(1)
\end{align*}
Therefore, 
\begin{align*}
&\E\left[\sum_{i=1}^T\frac{1}{2}(f_t(\wh_t^1) + f_t(\wh^t_1)) - \sum_{t=1}^T f_t(\w_t^*)\right]\\
&\leq O\left(\max(d^2\sqrt{S_T\max(1, B_T)}, d^{3/2}\max(1, B_T))\right)
\end{align*}

\subsection{Proof of Lemma~\ref{lem:11}}
From the Lipschitz property, 
\begin{align*}
f_t(\wh_t^1) - f_t(\w_t)&\leq G\|\wh_t^1 - \w_t\|_2\leq G\delta\\
f_t(\wh_t^2) - f_t(\w_t)&\leq G\|\wh_t^2- \w_t\|_2\leq G\delta
\end{align*}
Then 
\begin{align*}
\sum_{i=1}^T\frac{1}{2}(f_t(\wh_t^1) + f_t(\wh^t_1)) - \sum_{t=1}^T f_t(\w_t)\leq G\delta T
\end{align*}
To proceed, 
\begin{align*}
&f_t((1-\xi)\w_t^*)\leq \xi f_t(0) + (1-\xi)f_t(\w_t^*) \\
&= f_t(\w_t^*) + \xi(f_t(0) - f_t(\w_t^*))\leq f_t(\w_t^*) + \xi G\|\w_t^*\|_2\\
& = f_t(\w_t^*) + \xi GR_2
\end{align*}
Then 
\begin{align*}
\sum_{t=1}^T f_t((1-\xi)\w_t^* ) - \sum_{t=1}^Tf_t(\w_t^*)\leq \xi GR_2 T
\end{align*}
Thus
\begin{align*}
&\sum_{i=1}^T\frac{1}{2}(f_t(\wh_t^1) + f_t(\wh^t_1)) - \sum_{t=1}^T f_t(\w_t^*)\\
&\leq \sum_{t=1}^Tf_t(\w_t) - \sum_{t=1}^Tf_t((1-\xi)\w_t^*) + G\delta T + \xi GR_2 T\\
&\leq  \sum_{t=1}^Tf_t(\w_t) - \sum_{t=1}^Tf_t((1-\xi)\w_t^*)  + O(1)
\end{align*}

\end{document}